\documentclass[journal]{IEEEtran}
\usepackage{graphicx}
\usepackage{amsmath}
\usepackage[linesnumbered,ruled,vlined]{algorithm2e}
\usepackage{booktabs}
\newtheorem{theorem}{Theorem}

\newtheorem{proof}{Proof}
\usepackage{multirow}
\ifCLASSINFOpdf
\else
\fi
\usepackage{amsmath}
\usepackage{amsfonts}       
\usepackage{xcolor}
\usepackage{tikz}
\usetikzlibrary{arrows.meta, shapes.geometric, positioning}

\begin{document}
%
\title{Enhanced High-Dimensional Data Visualization through Adaptive Multi-Scale Manifold Embedding}
%
%
%

\author{Tianhao Ni\textsuperscript{*}, Bingjie Li\textsuperscript{*}, and Zhigang Yao
\thanks{*These authors contributed equally to this work.}
\thanks{This work was supported by the Singapore Ministry of
Education Tier 2 grant A-8001562-00-00 and the Tier 1 grant A-8002931-00-00 at the National
University of Singapore.} 
\thanks{T. Ni is with the School of Mathematical Sciences, Zhejiang University, Hangzhou, China (e-mail: thni@zju.edu.cn).}
\thanks{B. Li and Z. Yao are with the Department of Statistics and Data Science, National University of Singapore (e-mail: bjlistat@nus.edu.sg and zhigang.yao@nus.edu.sg).}}

%
%

\markboth{Journal of \LaTeX\ Class Files,~Vol.~14, No.~8, August~2015}%
{Shell \MakeLowercase{\textit{et al.}}: Bare Demo of IEEEtran.cls for IEEE Journals}
%



\maketitle

\begin{abstract}
To address the dual challenges of the curse of dimensionality and the difficulty in separating intra-cluster and inter-cluster structures in high-dimensional manifold embedding, we proposes an Adaptive Multi-Scale Manifold Embedding (AMSME) algorithm. By introducing ordinal distance to replace traditional Euclidean distances, we theoretically demonstrate that ordinal distance overcomes the constraints of the curse of dimensionality in high-dimensional spaces, effectively distinguishing heterogeneous samples. We design an adaptive neighborhood adjustment method to construct similarity graphs that simultaneously balance intra-cluster compactness and inter-cluster separability. Furthermore, we develop a two-stage embedding framework: the first stage achieves preliminary cluster separation while preserving connectivity between structurally similar clusters via the similarity graph, and the second stage enhances inter-cluster separation through a label-driven distance reweighting. Experimental results demonstrate that AMSME significantly preserves intra-cluster topological structures and improves inter-cluster separation on real-world datasets. Additionally, leveraging its multi-resolution analysis capability, AMSME discovers novel neuronal subtypes in the mouse lumbar dorsal root ganglion scRNA-seq dataset, with marker gene analysis revealing their distinct biological roles.
\end{abstract}

\begin{IEEEkeywords}
Manifold Embedding, Scale-invariant Metric, Curse of Dimension, Adaptive Neighborhood Identification, Visualization, Multi-Resolution Analysis. 
\end{IEEEkeywords}

\IEEEpeerreviewmaketitle

\section{Introduction}
Manifold embedding has emerged as a pivotal tool in scientific research, encompassing data-driven disciplines such as machine learning \cite{mcconville2021n2d,sainburg2021parametric} and computational social science \cite{sharifian2022analysing}, as well as traditional domains including physics \cite{haggar2024reconsidering}, chemistry \cite{trozzi2021umap}, and biology \cite{becht2019dimensionality,kobak2019art}. Researchers frequently encounter datasets comprising thousands or even millions of variables, necessitating methodologies to extract core patterns, identify clusters or submanifolds, and generate interpretable low-dimensional representations. These representations facilitate exploratory data analysis, hypothesis generation, anomaly detection, and the intuitive communication of complex results.

Over the past two decades, manifold embedding methods have made substantial progress. Early linear techniques, such as Principal Component Analysis (PCA, \cite{hotelling1933analysis}), were introduced in the last century. In contrast, more sophisticated manifold learning frameworks gained prominence in the 2000s and 2010s. Techniques like Isomap \cite{tenenbaum2000global}, Laplacian Eigenmaps (LE, \cite{belkin2001laplacian}), Locally Linear Embedding (LLE, \cite{roweis2000nonlinear}), t-SNE \cite{van2008visualizing}, and more recent approaches such as UMAP \cite{mcinnes2018umap} and PACMAP \cite{wang2021understanding} have continuously enhanced the ability to preserve high-dimensional relationships in low-dimensional spaces. Additionally, manifold fitting techniques \cite{yao2023manifold,fefferman2023fitting,2019Manifold} have garnered attention for their capacity to reconstruct underlying manifold structures more effectively and handle noisy, non-uniform data distributions with greater robustness.

Despite their widespread adoption, existing manifold learning methods face significant challenges when data exhibit complex characteristics such as noise, high intra-cluster variability, or non-uniform density across different regions of the manifold. For instance, t-SNE and UMAP sometimes fail to separate distinct clusters due to inappropriate neighborhood scale settings \cite{van2018s}. Moreover, many traditional algorithms rely on absolute distances, which are highly sensitive in high-dimensional spaces and often lose their intuitive meaning due to the curse of dimensionality \cite{hammer1962adaptive}, making it difficult to learn the true structure of high-dimensional manifolds.

\begin{figure*}[htbp]
    \centering
\includegraphics[width=1\linewidth]{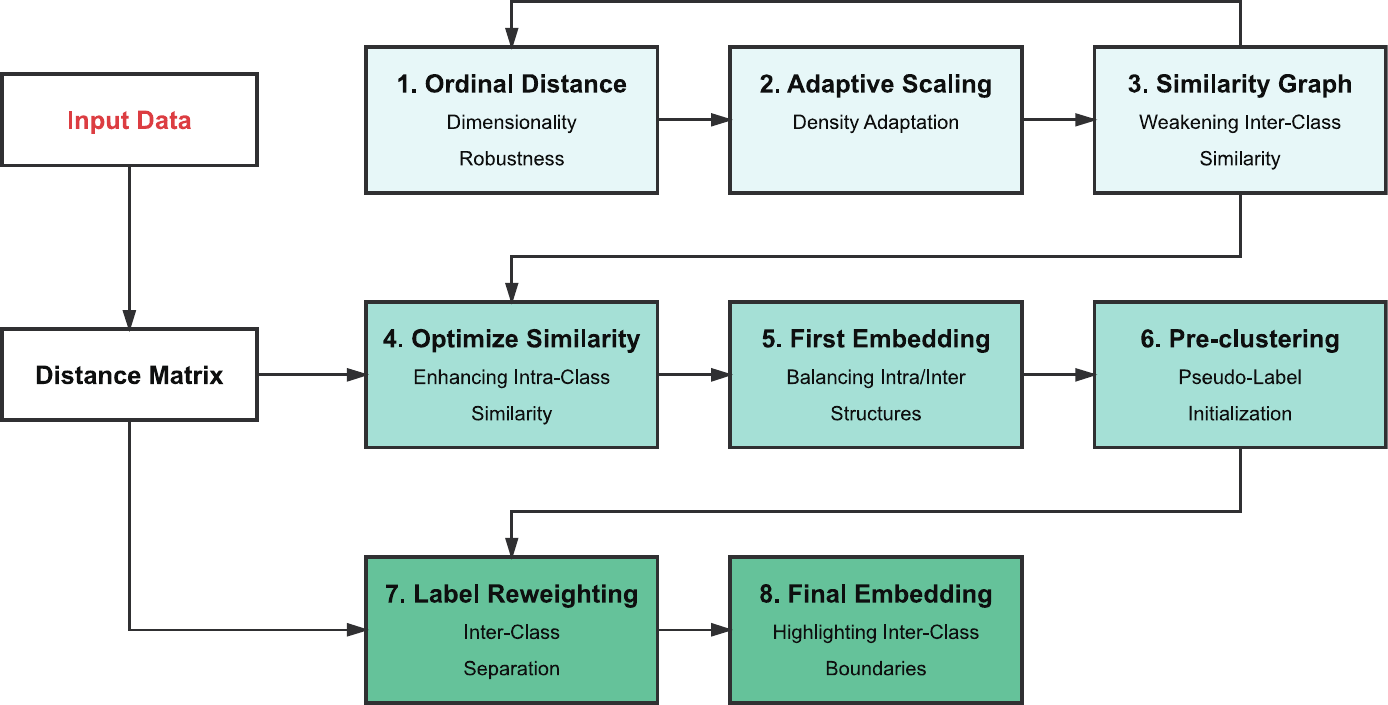} 
    \caption{Overview of the AMSME Framework (see Algorithm \ref{alg:AMSME} for detail). First, AMSME acquires the input data's distance matrix, then constructs an ordinal distance to overcome the curse of dimensionality. Subsequently, it adaptively selects neighborhood sizes based on density variations and builds a similarity graph to weaken inter-cluster similarities while enhancing intra-cluster cohesion. Based on this graph, AMSME performs the first visualization and obtains pseudo-labels via pre-clustering. Using these labels, it amplifies inter-cluster discrepancies in the distance matrix and conducts a second visualization with the updated matrix to achieve distinct inter-cluster separation.}
    \label{fig:workflow}
\end{figure*}

We propose a two-stage nonlinear manifold learning framework, termed Adaptive Multi-Scale Manifold Embedding (AMSME), to address these limitations. Our method advances manifold learning principles in the following ways:
\begin{itemize}
\item \textbf{Robustness via ordinal distances.}
AMSME replaces absolute distances with ordinal rankings, which theoretically and empirically demonstrate stable differentiation between heterogeneous and homogeneous samples in high dimensions.
\item \textbf{Adaptive local scaling.}
We adjust the effective neighborhood size of each sample point based on gap between ordinal distance. For high-density samples in different clusters, a larger neighborhood width effectively ensures strong intra-cluster connectivity while avoiding the introduction of inter-cluster connections. For samples located at the cluster center and cluster boundary, respectively, we adopt a differentiated strategy: assigning a smaller neighborhood width to the cluster center samples to minimize inter-cluster connections, while allocating a larger neighborhood width to the boundary samples to prevent them from being misidentified as outliers. This approach ensures the quality of manifold embedding is effectively preserved. Such an adaptive mechanism enhances the discernibility of global structures while preserving local structures.
\item \textbf{Multi-stage embedding.}
We propose a two-stage manifold embedding framework, which generates results customized to different visualization objectives. In the first stage, the embedding focuses on preserving intra-cluster structures. Although the distances between different clusters are relatively small, their boundaries remain clearly distinguishable. In the second stage, we leverage the label information from the first stage to drive the final embedding, further optimizing inter-cluster separability. This significantly reduces inter-cluster overlap while simultaneously strengthening the boundaries between clusters.
\end{itemize}

The remainder of this paper is organized as follows: Section \ref{sec:method} elaborates on the specific framework of AMSME, including the definition and theoretical analysis of ordinal distance, adaptive neighborhood selection, similarity graph construction, and two-stage embedding. Section \ref{sec:experiments} demonstrates the effectiveness of the proposed method by comparing AMSME with standard t-SNE \cite{van2008visualizing}, UMAP \cite{mcinnes2018umap}, and PaCMAP \cite{wang2021understanding} on real datasets. Finally, Section \ref{sec:conclusion} discusses the feasibility of adaptive multi-scale embedding in other types of data analysis.

\section{Methodology}
\label{sec:method}
Figure \ref{fig:workflow} illustrates the workflow of our proposed Adaptive Multi-Scale Manifold Embedding (AMSME) framework, which comprises five key steps: (1) Construction of the ordinal distance matrix, (2) Adaptive neighborhood identification and similarity graph construction, (3) First-stage embedding and clustering, (4) Label-driven reweighting and final embedding.

\subsection{Notations}
Let \(X = [x_1, \dots, x_n] \in \mathbb{R}^{d \times n}\) denote the dataset comprising \(n\) samples of dimensionality \(d\), and let \(D \in \mathbb{R}^{n \times n}\) represent its corresponding Euclidean distance matrix and $I_d$ denote the d-dimensional identity matrix. We denote the multivariate normal distribution parameterized by mean \(\mu\) and covariance \(\Sigma\) as \(\mathcal{N}(\mu, \Sigma)\), and use \(\mathbb{P}(\cdot)\), \(\mathbb{E}(\cdot)\), and \(\text{Var}(\cdot)\) to represent the probability measure, expectation, and variance operators, respectively. The asymptotic order notation \(\mathcal{O}(\cdot)\) quantifies the growth rates of functions. The embedding map \(\mathcal{F}: D \mapsto Y \in \mathbb{R}^{k \times n}\) transforms pairwise distance matrices into k-dimension representations where $k\ll d$, with UMAP as the default method. For a matrix \(A\), \(A_{i,:}\) and \(A_{:,j}\) denote its \(i\)-th row and \(j\)-th column vectors, respectively, and \(A_{i,j}\) represents the \((i,j)\)-th element. Additionally, \(\|\cdot\|\) denotes the Euclidean norm of a vector.

\subsection{Ordinal Distance}
In high-dimensional data analysis, the Euclidean distance is significantly affected by the curse of dimensionality, making it unreliable for measuring inter-cluster differences. However, we discovered that the relative magnitude of distances can effectively distinguish clusters in high-dimensional space, as demonstrated by Theorem \ref{thm:curse}.

\begin{theorem}  
\label{thm:curse}  
Let \(x_i, x_j \sim \mathcal{N}(\mu_1, \sigma_1^2 I_d)\) be independently and identically distributed for \(i \neq j\), and \(y_k \sim \mathcal{N}(\mu_2, \sigma_2^2 I_d)\). If the global separability condition \(\sigma_2^2 - \sigma_1^2 + \|\mu_1 - \mu_2\|^2 > 0\) holds, define the intra-cluster squared distance \(d_{ij} = \|x_i - x_j\|^2\) and the inter-cluster squared distance \(d_{ik} = \|x_i - y_k\|^2\). Then, as the dimension \(d \to \infty\), the probability that the inter-cluster distance exceeds the intra-cluster distance converges to 1:  
\[  
\lim_{d \to \infty} \mathbb{P}(d_{ik} > d_{ij}) = 1-\lim\limits_{d\to\infty}\mathcal{O}(d^{-1}) = 1.  
\]  
\end{theorem}  

\begin{proof}
For any pair \((x_i, x_j)\), the difference vector \(x_i - x_j\) follows a zero-mean Gaussian distribution as 
\[
x_i - x_j \sim \mathcal{N}(0, 2\sigma_1^2 I_d).
\]  
The squared intra-cluster distance \(d_{ij} = \|x_i - x_j\|^2\) is therefore a sum of \(d\) independent squared Gaussian variables, yielding a scaled chi-squared distribution
\[
d_{ij} \sim 2\sigma_1^2 \chi^2(d),
\]  
with mean \(\mathbb{E}[d_{ij}] = 2d\sigma_1^2\) and variance \(\mathrm{Var}(d_{ij}) = 8d\sigma_1^4\).  

For inter-cluster distances, the difference vector \(x_i - y_k\) combines the statistical properties of both classes. Since \(x_i\) and \(y_k\) are independent, we have
\[
x_i - y_k \sim \mathcal{N}\left(\mu_1 - \mu_2, (\sigma_1^2 + \sigma_2^2)I_d\right).
\]  
The squared inter-cluster distance \(d_{ik} = \|x_i - y_k\|^2\) thus follows a non-central chi-squared distribution
\[
d_{ik} \sim (\sigma_1^2 + \sigma_2^2) \chi^2\left(d, \lambda = \frac{d\|\mu_1 - \mu_2\|^2}{\sigma_1^2 + \sigma_2^2}\right),
\]  
where \(\lambda\) is the non-centrality parameter. Its mean and variance are
\begin{align*}
    \mathbb{E}[d_{ik}] &= d\left(\sigma_1^2 + \sigma_2^2 + \|\mu_1 - \mu_2\|^2\right),\\
    \mathrm{Var}(d_{ik}) &= 4d\|\mu_1 - \mu_2\|^2(\sigma_1^2 + \sigma_2^2) + 2d(\sigma_1^2 + \sigma_2^2)^2.
\end{align*} 
The random variable \(Z\) measures the gap between inter-cluster and intra-cluster distances. Then the expectation of $Z$ is 
\[
\mathbb{E}[Z] = \mathbb{E}[d_{ik}] - \mathbb{E}[d_{ij}] = d\left(\sigma_2^2 - \sigma_1^2 + \|\mu_1 - \mu_2\|^2\right).
\]  
The variance of \(Z\) combines contributions from both distances
\begin{align*}
  \mathrm{Var}(Z)&=\mathrm{Var}(d_{ik})+\mathrm{Var}(d_{ij})+2\mathrm{Cov}(d_{ij},d_{ik}) \\
  &\leq 2(\mathrm{Var}(d_{ik})+\mathrm{Var}(d_{ij}))\\
  &= 8d\|\mu_1 - \mu_2\|^2(\sigma_1^2 + \sigma_2^2) + 4d(\sigma_1^2 + \sigma_2^2)^2 + 16d\sigma_1^4.
\end{align*}  
To bound \(\mathbb{P}(Z > 0)\), apply the Cantelli inequality \cite{cantelli1929sui}
\[
\mathbb{P}(Z > 0) \geq 1 - \frac{\mathrm{Var}(Z)}{\mathrm{Var}(Z) + \mathbb{E}[Z]^2},
\]  
where the upper bound is a decreasing function with respect to \( \text{Var}(Z) \).
Substituting \(\mathbb{E}[Z]\) and \(\mathrm{Var}(Z)\)
\[
\mathbb{P}(Z > 0) \geq 1 - \frac{\mathcal{O}(d)}{\left[d\left(\sigma_2^2 - \sigma_1^2 + \|\mu_1 - \mu_2\|^2\right)\right]^2 + \mathcal{O}(d)}.
\]  
As \(d \to \infty\), the numerator scales as \(\mathcal{O}(d)\), while the denominator grows as \(\mathcal{O}(d^2)\). Thus
\[
1\ge \lim_{d \to \infty} \mathbb{P}(Z > 0) \geq 1 - \lim_{d \to \infty} \frac{\mathcal{O}(d)}{\mathcal{O}(d^2)} = 1.
\]  
This implies
\[
\lim_{d \to \infty} \mathbb{P}(Z > 0) = 1.
\]  
\end{proof}

Theorem \ref{thm:curse} demonstrates that when there are significant variance differences or mean discrepancies between clusters, although the absolute distance differences may not be pronounced, the relative magnitude of distances can stably distinguish heterogeneous samples from homogeneous ones. This discriminative capability strengthens progressively as the dimensionality increases. This mechanism overcomes the failure of traditional distance metrics in high-dimensional scenarios, where absolute distance measures typically lose their discriminative power.  

\begin{figure}[ht]
    \centering
\includegraphics[width=\linewidth]{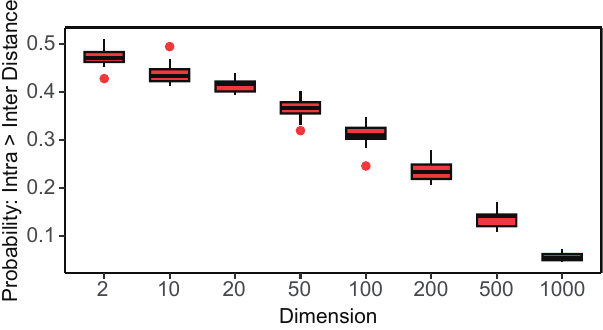} 
    \caption{Probability of intra-cluster Distance Exceeding inter-cluster Distance Based on 10 Repeated Trials.}
    \label{fig:curse}
\end{figure}

\begin{figure*}[htbp]
    \centering
\includegraphics[width=1\linewidth]{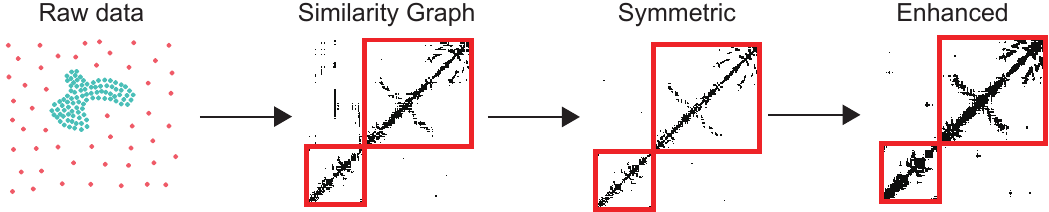} 
    \caption{The comparison of results from the three-step similarity graphs.}
    \label{fig:similarity_graph}
\end{figure*}

To validate the theoretical predictions of Theorem \ref{thm:curse}, we conducted a series of numerical experiments. In these experiments, we generated two clusters of Gaussian-distributed datasets, \(X_1\) and \(X_2\), both with zero means and standard deviations \(\sigma_1 = 1.0\) and \(\sigma_2 = 1.1\), respectively. For dimensions \(d \in \{2, 10, 20, 50, 100, 200, 500, 1000\}\), we computed the probability that the intra-cluster distance \(d_{ij}\) within \(X_1\) exceeds the inter-cluster distance \(d_{ik}\) between samples from \(X_1\) and \(X_2\). As shown in Figure \ref{fig:curse}, the experimental results demonstrate that as the dimension \(d\) increases, the probability of intra-cluster distances exceeding inter-cluster distances decreases significantly. This phenomenon is fully consistent with the theoretical prediction of Theorem \ref{thm:curse}. These results confirm the effectiveness of ordinal relationships in mitigating the curse of dimensionality. Specifically, by capturing relative ranking relationships, we can stably distinguish heterogeneous samples from homogeneous ones, and its discriminative capability progressively strengthens with increasing dimensionality.

To overcome the curse of dimension, we introduce an ordinal distance based on relative magnitude relationships \cite{li2022robust}. The core idea is to replace absolute distance comparisons with local ranking relationships. The ordinal distance between samples \(x_i\) and \(x_j\) is defined as the ranking position of \(x_j\) in the neighborhood of \(x_i\):  
\begin{align}  
\label{equ:def_order_dist}  
o(x_i; x_j) = \text{card}\left(\{k \mid D_{i,k} < D_{i,j}, 1 \leq k \leq n\}\right),  
\end{align}  
where \(\text{card}(\cdot)\) denotes the cardinality of a set.

The ordinal distance is also robust to noise. Specifically, the probability of changes in the ranking of distances is linearly related to the data dimensionality and the variance of the noise, as detailed in Theorem \ref{thm:noise}.
\begin{theorem}
\label{thm:noise}
Let the original data matrix be \( X = (x_1, \dots, x_n) \in \mathbb{R}^{d \times n} \) with euclidean distance matrix \( D \). The noise matrix \( E = (e_1, \dots, e_n) \) satisfies \( e_i \sim \mathcal{N}(0, \sigma^2 I_d) \) with independent \( e_i, e_j \). The perturbed data matrix \( X' = X + E \) with the euclidean distance matrix $D'$. Then for any neighboring pairs with \( D_{i,j}^2 < D_{i,k}^2 \), we have 
\[
P\left( (D'_{i,j})^2 > (D'_{i,k})^2 \right) \leq \frac{16d\sigma^2(D_{i,j}^2 + D_{i,k}^2)}{(D_{i,k}^2-D_{i,j}^2)^2} + \mathcal{O}(\sigma^3).
\]
\end{theorem}
\begin{proof}[Proof of Theorem \ref{thm:noise}]
Define the squared distance perturbation:
\[
\eta_{ij} := (D'_{i,j})^2 - D_{i,j}^2 = 2(x_i - x_j)^\top(e_i - e_j) + \|e_i - e_j\|^2.
\]
Let \( f_{ijl} = e_{i,l} - e_{j,l} \sim \mathcal{N}(0, 2\sigma^2) \) denote each component of \( e_i - e_j \). Then
\[
\mathbb{E}[\eta_{ij}] = \mathbb{E}\left[\sum_{l=1}^d f_{ijl}^2\right] = \sum_{l=1}^d 2\sigma^2 = 2d\sigma^2.
\]
For variance, expand \( \eta_{ij}^2 \) as
\begin{align*}
    &\eta_{ij}^2 = 4\left(\sum_{l=1}^d z_{ijl}f_{ijl}\right)^2 \\
    &+ \left(\sum_{l=1}^d f_{ijl}^2\right)^2 + 4\left(\sum_{l=1}^d z_{ijl}f_{ijl}\right)\left(\sum_{l=1}^d f_{ijl}^2\right),
\end{align*}
where \( z_{ijl} = x_{i,l} - x_{j,l} \).

Compute expectations term-wise, we have
\begin{align*}
&\mathbb{E}\left[\left(\sum z_{ijl}f_{ijl}\right)^2\right] = \sum_{l=1}^d z_{ijl}^2 \mathbb{E}[f_{ijl}^2] = 2d\sigma^2\|x_i - x_j\|_2^2, \\
&\mathbb{E}\left[\left(\sum f_{ijl}^2\right)^2\right] = 4d\sigma^4 + 8d^2\sigma^4, \\
&\mathbb{E}\left[\left(\sum z_{ijl}f_{ijl}\right)\left(\sum f_{ijl}^2\right)\right] = 0.
\end{align*}
Thus,
\[
\mathrm{Var}(\eta_{ij}) = 8d\sigma^2\|x_i - x_j\|_2^2 + \mathcal{O}(\sigma^3).
\]

The event \( (D'_{i,j})^2 > (D'_{i,k})^2 \) is equivalent to \( \eta_{i,j} - \eta_{i,k} > D_{i,k}^2-D_{i,j}^2=\Delta_{ijk} \). By Chebyshev's inequality, we have 
\begin{align*}
&P(\eta_{ij} - \eta_{ik} > \Delta_{ijk})
=P(|\eta_{ij}-\eta_{ik}|>\Delta_{ijk})\\
\leq &\frac{\mathrm{Var}(\eta_{ij} - \eta_{ik})}{\Delta_{ijk}^2} \\
= &\frac{16d\sigma^2(\|x_i - x_j\|_2^2 + \|x_i - x_k\|_2^2)}{\Delta_{ijk}^2}+\mathcal{O}(\sigma^3).
\end{align*}
\end{proof}

To ensure the symmetry of the ordinal distance matrix and to amplify the ordinal discrepancies both within and between cluters, we further define a symmetrized ordinal distance matrix \(O \in \mathbb{R}^{n \times n}\), whose elements satisfy
\begin{align}
\label{equ:symm}
    O_{i,j} = \max\{o(x_j; x_i), o(x_i; x_j)\}.
\end{align}

\subsection{Adaptive Neighborhood Identification and Similarity Graph Construction}
To enhance the consistency among samples of the same cluster and the distinction between samples of different clusters, we construct a similarity graph based on the ordinal distance matrix \(O \in \mathbb{R}^{n \times n}\), enabling nonlinear modeling of pairwise relationships. The similarity matrix \(A \in \mathbb{R}^{n \times n}\) is computed using a Gaussian-like kernel as 
\begin{align}
    A_{i,j} = \exp\left(-\frac{O_{i,j}^2}{\sigma_{i,j}^2}\right),
\end{align}
where \(\sigma_{i,j}\) denotes the adaptive neighborhood scale that controls the decay range of similarity.

To achieve finer modeling of regions with varying density, AMSME dynamically determines the neighborhood scale for each sample by analyzing potential gaps in the sorted entries of \(O\). First, we select the upper bound of the neighborhood size \(k\) based on the fact that, for \(n\) data points drawn independently and identically distributed (i.i.d.) from a density function with connected support, the \(k\)-nearest neighbor graph and the mutual \(k\)-nearest neighbor graph are connected if \(k\) is chosen on the order of \(\log(n)\) \cite{brito1997connectivity,von2007tutorial}. Assuming an approximately equal number of samples in each cluster, we set \(k\) as
\begin{align*}
    k = 2 \max\left(\lfloor \log\left(\tfrac{2n}{n_c}\right)\rfloor, 3\right),
\end{align*}
where \(n_c\) denotes the number of clusters.  

For each row \(O_{i,:}\) of the matrix \(O\), we extract the \(k\)-smallest values as a vector \(M^i\). We then compute the differences between consecutive elements of \(M_i\) as  
\begin{align}
\label{equ:F^i}
    F^i_{j} = M^i_{j+1} - M^i_{j}, \quad j = 1, \dots, k-1.
\end{align}
Next, we identify the maximum difference \(a^i\) and its corresponding index \(b^i\)
\begin{align*}
a^i = \max_j F^i_{j}, \quad b^i = \arg\max_j F^i_{j}.
\end{align*}

When a significant density gap is detected \(a^i > 1\), the local neighborhood size \(s^i\) for sample \(i\) is defined as
\begin{align*}
    s^i = \max\left(b^i, \tfrac{k}{2}-1\right).
\end{align*}
Otherwise, the neighborhood size is set to \(k-1\).  

Finally, based on the determined neighborhood size, the kernel bandwidth parameter for sample \(i\) with respect to sample \(j\) is defined as $\sigma_{i,j} = M^i_{s^j}$.

Our design of \(\sigma\) is based on the impact of sample density on ordinal distances. When both \(x_i\) and \(x_j\) are high-density samples located near the cluster center, their ordinal distances do not exhibit significant gaps. In this case, we set both \(\sigma_{i,j}\) and \(\sigma_{ji}\) to relatively large values, ensuring that all high-density samples within the cluster center are included in their confidence neighborhoods, thereby maintaining the tight connectivity of the cluster. Conversely, when \(x_i\) is a high-density sample near the cluster center and \(x_j\) is a low-density sample at the cluster boundary, the ordinal distances satisfy \(o(x_i;x_j) > o(x_j;x_i)\). After symmetrization, this results in a gap in the ordinal distances from \(x_j\) to other samples, while no such gap exists for \(x_i\). We set \(\sigma_{i,j}\) to a smaller value to prevent the high-similarity neighborhood of \(x_i\) from including boundary samples, while setting \(\sigma_{ji}\) to a larger value to ensure that the low-density sample \(x_j\) maintains sufficient similarity with other medium-density regions, avoiding the isolation of low-density regions. This design enables \(\sigma\) to dynamically adapt to changes in sample density, preserving fine-grained local structures in high-density regions while maintaining global connectivity between sparse regions and the cluster center.

Subsequently, we apply a symmetrization operation as
\begin{align*}
A \;\leftarrow\; \min\bigl(A,\,A^T\bigr).
\end{align*}
This symmetrization not only ensures intra-cluster connectivity but also effectively reduces inter-cluster connections, enhancing the representational capability of the similarity matrix. To further strengthen the similarity between samples within the same cluster, we introduce a secondary connection strategy to enhance the similarity further as 
\begin{align*}
    S = \min(1,A^2).
\end{align*}
The results of the three-step similarity graph construction on the Compounded dataset \footnote{http://cs.joensuu.fi/sipu/datasets/} are shown in Figure \ref{fig:similarity_graph}. Before symmetrization, the constructed similarity graph exhibits strong intra-cluster connections but introduces a small number of inter-cluster connections (primarily located in the upper-left corner). The symmetrization operation effectively removes inter-cluster connections while preserving intra-cluster connections. Finally, the secondary connection step further enhances intra-cluster connections with almost no involvement of inter-cluster connections.

\begin{algorithm}[t]
\SetAlgoLined
\KwIn{%
Distance Matrix $D \in \mathbb{R}^{n\times n}$, number of clusters $n_c$, constant $\alpha>1$}
\KwOut{%
First embedding $Y_1 \in \mathbb{R}^{k\times n}$ and Final embedding $Y_2 \in \mathbb{R}^{k\times n}$}

Construct ordinal distance matrix as \eqref{equ:def_order_dist} and symmetrization by \eqref{equ:symm} as $O$.\\
Let $k=2\cdot\max(\lfloor \log(2n/n_c)\rfloor,\,3)$. For each $i$, find the $k$ smallest values of $O_{i,:}$ as $M^i$, compute their differences $F^i$ as \eqref{equ:F^i}, and locate the largest gap $a^i$ with index $b^i$. Define 
\[
  s^i=
  \begin{cases}
    \max(b^i,k/2-1), & \text{if } a^i>1,\\
    k-1, & \text{otherwise},
  \end{cases}
\]
and set $\sigma_{i,j}=M^i_{s^j}$.\\
Form similarity $A_{i,j}=\exp[-O_{i,j}^2/\sigma_{i,j}^2]$ and symmetrize $A \leftarrow \min(A,A^T)$ 
and further enhanced to $S = \min(1,A^2)$.\\
 Run UMAP on $D^O=1-S$ to obtain an initial embedding $Y_1$.\\
Cluster $Y_1$ (e.g., K-means) to get labels $\ell_1$.\\
Adjust $D$ into $D^M$ by normalizing intra-cluster distances within $[0,1]$, and assigning $D^M=\alpha$ for inter-cluster pairs by label $\ell_1$.\\
Run UMAP again on $D^M$ to produce final embedding $Y_2$.

\caption{\textbf{Adaptive Multi-Scale Manifold Embedding (AMSME)}}
\label{alg:AMSME}
\end{algorithm}

\subsection{First Embedding and Clustering}
To further preserve local neighborhood information, enhance the separation of dissimilar samples, and mitigate the impact of noise, \(D^O = 1-S\) is used as the input to the embedding algorithm to better capture the nonlinear structure of high-dimensional data, yielding an intermediate low-dimensional layout $Y_1 = \mathcal{F}(D^O)\in\mathbb{R}^{2\times n}$.

Through the coupling of non-linear steps that enhance local structures, the embedding result \(Y_1\) clusters similar samples while forming clear boundaries between classes. At this stage, a clustering algorithm (e.g. K-means) can be applied to label each sample, generating pseudo-label \(\ell_1\) as label based on the first embedding result. The discovered labels become a guiding signal for emphasizing cluster boundaries in the subsequent visualization. 
\subsection{Label-Driven Reweighting and Final Embedding}
To enhance cluster separability while preserving local topology structures, AMSME modifies the original distance matrix \(D\) based on pseudo-label \(\ell_1\). If \(c_i\) denotes the set of samples belonging to the \(i\)-th cluster, the intra-cluster distances are normalized to the range \([0,1]\):  
\[
D^M_{c_i,c_i} = 
\frac{D_{c_i,c_i}}
{\max\bigl(D_{c_i,c_i}\bigr)}.
\]  
For samples belonging to different clusters, a large constant \(\alpha \in [1,\infty]\) (defaulting to \(2\)) is assigned to explicitly separate these groups. Finally, \(D^M\) is fed into embedding algorithm again to obtain the final embedding \(Y_2 = \mathcal{F}(D^M)\).

AMSME generates two distinct embedding results, denoted as AMSME-S1 (\(Y_1\)) and AMSME-S2 (\(Y_2\)), each tailored to address specific embedding objectives. The primary objective of AMSME-S1 is to produce clustering results that align with the real labels, ensuring that samples within the same cluster are tightly grouped while clearly delineating the internal structure of each cluster. This approach effectively captures subtle variations among samples within the same cluster. Although the distances between different clusters may remain relatively small, AMSME-S1 successfully preserves the relative proximity between clusters. In contrast, AMSME-S2 prioritizes achieving clear separation between distinct clusters, thereby accentuating the independence of each cluster. This dual focus allows AMSME to balance global inter-cluster relationships with local intra-cluster structures, offering a comprehensive and multifaceted framework for the analysis and interpretation of high-dimensional data.

\begin{figure*}[htbp]
    \centering
\includegraphics[width=\textwidth]{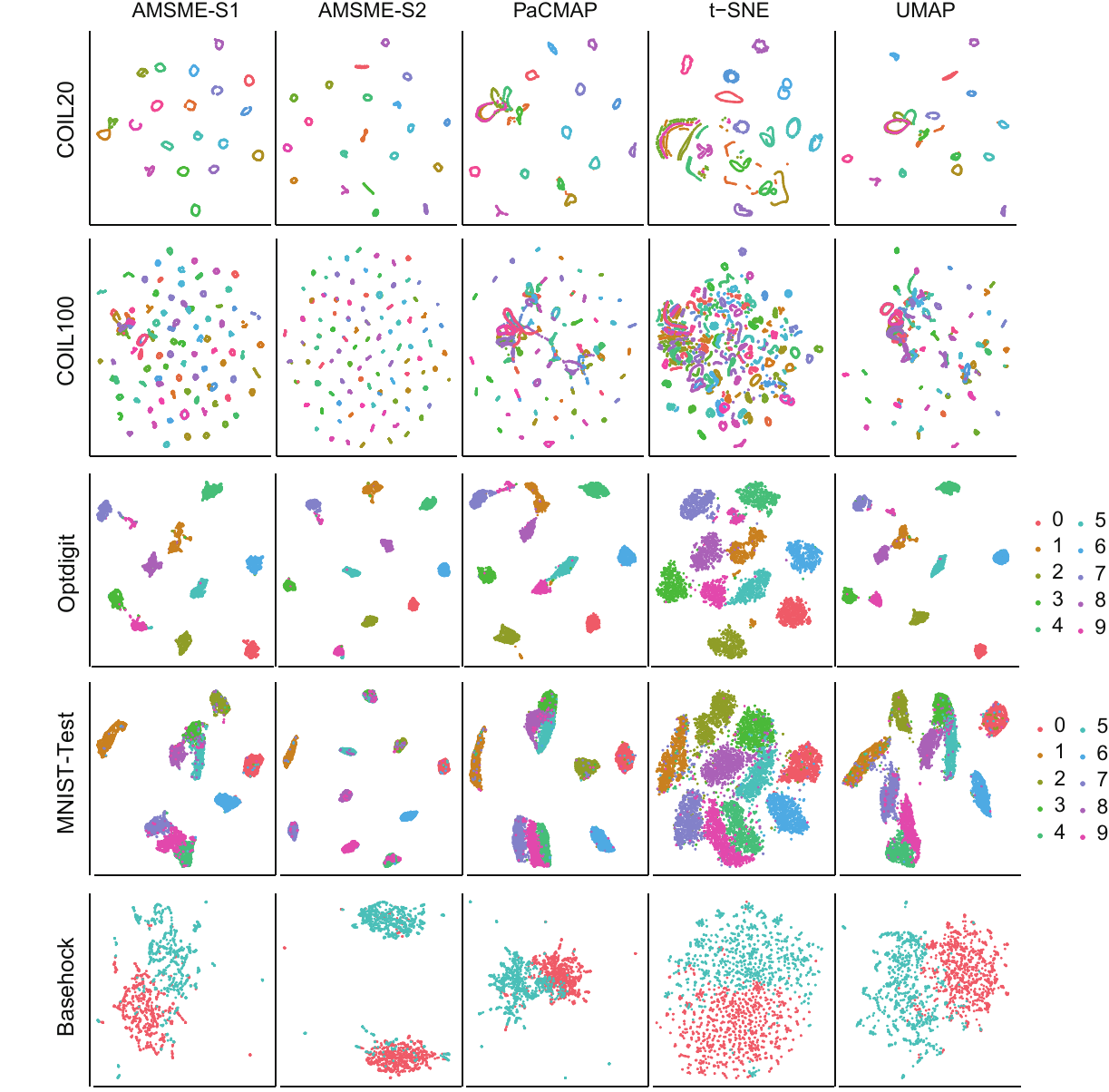} 
    \caption{Manifold embedding results for five datasets using five methods.}
    \label{fig:embedding}
\end{figure*}

\section{Experimental Analysis and Results}
\label{sec:experiments}
In our experiments, we evaluated our algorithm on several benchmark datasets, encompassing diverse image datasets (such as COIL20 \cite{nene1996columbia}, COIL100 \cite{nene1996columbia}, Optdigit\footnote{https://archive.ics.uci.edu/dataset}, and MNIST-Test \cite{lecun1998mnist}) and the text dataset Basehock\footnote{http://qwone.com/~jason/20Newsgroups/}. Detailed descriptions of these datasets are provided in Table \ref{tab:datasets}. To assess the effectiveness of AMSME, we conducted a comparative analysis with several widely used state-of-the-art embedding algorithms, including t-SNE \cite{van2008visualizing}, UMAP \cite{mcinnes2018umap}, and PACMAP \cite{wang2021understanding}.

\begin{table}[ht]
\centering
\caption{Summary of benchmark datasets.}
\label{tab:datasets}
\begin{tabular}{lccc}
\hline
\textbf{Dataset} & \textbf{\#Samples} & \textbf{\#Features} & \textbf{\#Classes} \\
\hline
COIL20  & 1,440 & 1,024 & 20 \\
COIL100  & 7,200 & 1,024 & 100 \\
Optdigit & 5620 & 64 & 10\\
MNIST-Test & 10,000 & 784 & 10 \\
Basehock  & 1,993 & 4862 & 2\\
\hline
\end{tabular}
\end{table}
\subsection{Experimental Setting}
In our experiments, all comparative algorithms were executed using their default parameter settings. For the t-SNE algorithm, the perplexity parameter was set to 30. For the UMAP algorithm, the default neighborhood size was set to 15, and the minimum distance was set to 0.1. The PaCMAP algorithm utilized its default neighborhood settings. For distance computation, cosine similarity was applied to the Basehock dataset, while Euclidean distance was employed for all other datasets. To ensure the consistency of the embedding results, all methods utilized Principal Component Analysis (PCA) to reduce the original data to two dimensions as the initial embedding. Additionally, a fixed random seed was used to guarantee the reproducibility of the experiments, and the dimensionality of the manifold embedding was set to 2 for visualizing the differences between the results of these algorithms.

\begin{figure*}[t]
    \centering
\includegraphics[width=\textwidth]{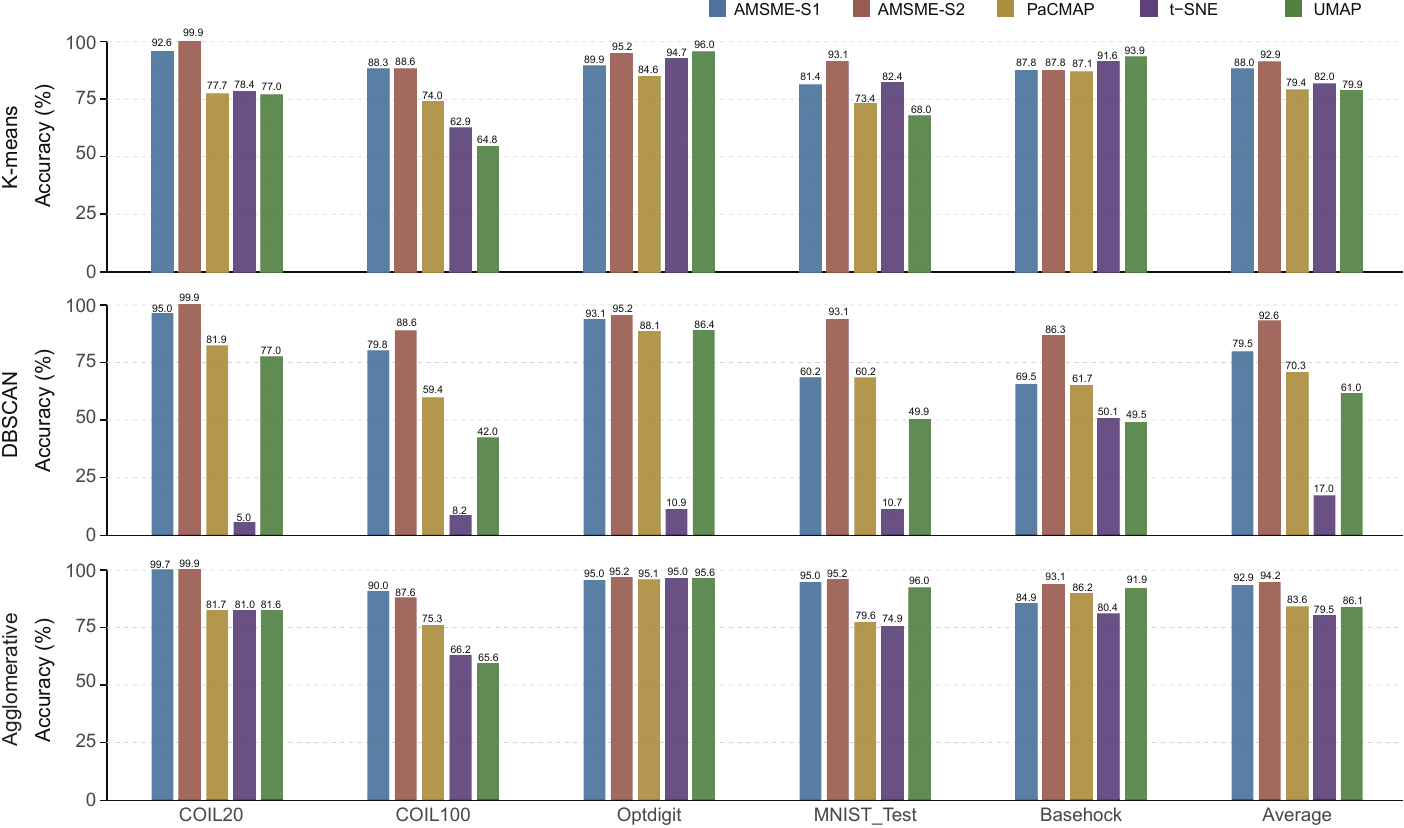} 
    \caption{Bar chart of ACC results for three clustering algorithms applied to visualization results on five datasets.}
    \label{fig:ACC}
\end{figure*}

\subsection{Comparison Results}
We first present the visualization results for the five datasets, as shown in Figure \ref{fig:embedding}. On the COIL20 and COIL100 datasets, each class consists of images of the same object captured from 72 different angles. The ideal visualization shape should be circular or figure-8 (reflecting the symmetric structure of the object). On COIL20, AMSME successfully preserves the topological structure within classes and the separation between classes, demonstrating its ability to accurately capture intra-cluster structures. In contrast, PaCMAP and UMAP exhibit overlapping between multiple classes, while t-SNE fails to recognize the intra-cluster topological structure. On COIL100, compared to the competing algorithms, AMSME shows significantly fewer instances of inter-cluster crossing. On the handwritten digit datasets Optdigit and MNIST-Test, AMSME-S2 successfully separates all digits, while AMSME-S1 identifies similarities between digits 4, 7, and 9, as well as between digits 3, 5, and 8 in the MNIST-Test dataset. This indicates that AMSME-S1 recognizes similarities between different clusters while maintaining clear boundaries between them, enabling AMSME-S2 to achieve complete cluster separation. On the text dataset Basehock, AMSME-S1 exhibits distinct boundaries between the two clusters.

To further validate the superiority of AMSME in inter-cluster separation performance, we employed three clustering algorithms—K-means \cite{lloyd1982least}, DBSCAN \cite{kriegel2011density}, and hierarchical clustering \cite{nielsen2016hierarchical}—and conducted a quantitative analysis using clustering accuracy (ACC) as the evaluation metric. Clustering Accuracy (ACC) is a metric used to evaluate the performance of clustering algorithms. It measures the extent to which the clusters produced by the algorithm match the ground truth labels of the data. ACC is typically calculated as the ratio of correctly assigned data points to the total number of data points, expressed as a percentage. The experimental results (as shown in Figure \ref{fig:ACC}) demonstrate that both stages of AMSME exhibit significant performance advantages across all five benchmark datasets. Specifically, under the K-means framework, the two stages of AMSME achieved 6\% and 10.9\% improvements in clustering accuracy compared to the second-best method, respectively. Particularly on the COIL20 and COIL100 datasets, the visualization results of AMSME displayed optimal inter-cluster separation. Notably, on the COIL20 dataset, the second stage of AMSME nearly achieved perfect clustering classification, which can be attributed to the reliable manifold embedding and clear cluster boundaries provided by the first stage. On the Optdigit and MNIST-Test datasets, the second stage of AMSME also performed exceptionally well, with clustering accuracy exceeding 93\%. Although AMSME's performance on the Basehock dataset under the K-means algorithm was slightly inferior to the best method, the gap was controlled within 6\%, still demonstrating strong competitiveness. It is worth noting that in the DBSCAN method, AMSME performed particularly well, while t-SNE completely failed. This phenomenon is primarily due to the fact that t-SNE's visualization results typically exhibit point cloud structures with uniform density, making it difficult to accurately delineate cluster boundaries based on density differences. In contrast, AMSME effectively addresses this issue through its unique similarity graph construction mechanism, further confirming its significant advantages in handling complex data structures.


\begin{figure*}[htbp]
    \centering
\includegraphics[width=\textwidth]{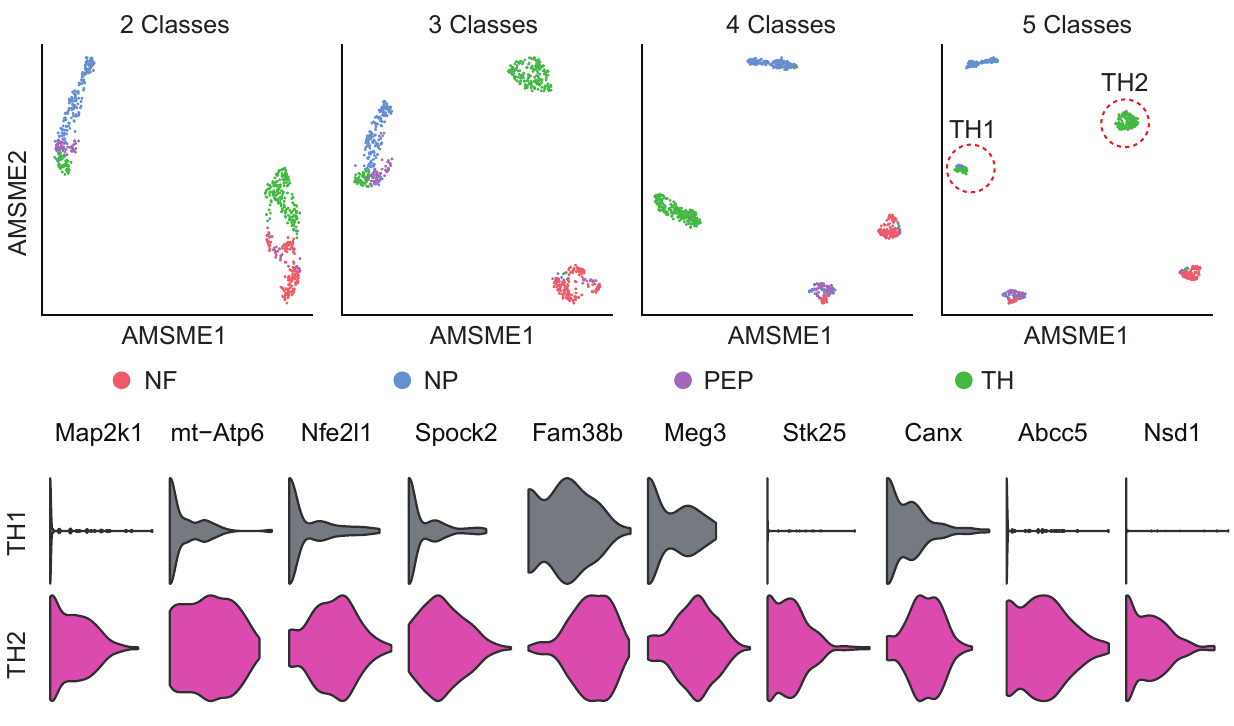} 
    \caption{Multi-resolution Analysis of GSE59739. The upper figure illustrates the visualization results of the second stage in AMSME, with the number of clusters set to 2, 3, 4, and 5, respectively (from left to right).
The lower figure presents violin plots of marker genes between the two subtypes of TH neurons.}
    \label{fig:Gene}
\end{figure*}

\subsection{Multi-resolution Analysis of Single-cell RNA Data}
AMSME demonstrates its multi-scale resolution capability on biological data by adjusting the number of clusters. We applied AMSME to the single-cell RNA sequencing dataset GSE59739 \cite{2015Unbiased} from the mouse lumbar dorsal root ganglion (DRG), which comprises four neuronal subtypes: neurofilament-enriched neurons (NF), neuropeptidergic neurons (NP), peptidergic nociceptors (PEP), and tyrosine hydroxylase-positive neurons (TH). The raw data were obtained from the GEO database and preprocessed using the Scanpy pipeline, including gene expression filtering, data normalization, and highly variable gene selection.

We systematically adjusted the clustering number from 2 to 5 to evaluate its hierarchical identification performance. When \(k=2\), the algorithm merged NF and TH into one cluster and NP and PEP into another, reflecting the macro-level functional division between sensory neurons (NF/TH) and nociceptive neurons (NP/PEP) in the DRG (Figure \ref{fig:Gene}). Increasing \(k\) to 3 successfully separated NF and TH due to their significant gene expression differences, while NP and PEP remained clustered (Figure \ref{fig:Gene}). Further setting \(k=4\), the algorithm accurately identified NF, NP, and TH neurons, although partial overlap persisted within the PEP subtype (Figure \ref{fig:Gene}).

When \(k=5\), AMSME first identified two functional subtypes of TH neurons, TH1 and TH2. The p-value for the Wilcoxon hypothesis test was 0, indicating significant transcriptional differences between the two subtypes. Differential expression analysis further identified 10 marker genes (Figure \ref{fig:Gene}), all of which were significantly upregulated in the TH2 subtype, highlighting distinct molecular profiles between TH1 and TH2.

The gene Map2k1 phosphorylates and activates ERK1 and ERK2 \cite{zhang2019identification}, which are essential for neuronal proliferation, survival, and neurogenesis \cite{sahu2021inhibition}. MT-ATP6 provides crucial information for the synthesis of proteins vital to mitochondrial function. Fam38b mediates in vivo calcium signaling in trigeminal ganglion neurons and electrophysiological signals in spinal dorsal horn neurons in response to non-noxious stimuli \cite{murthy2023deciphering}. GBF1 is involved in regulating COPI complex-mediated retrograde vesicular transport between the endoplasmic reticulum and the Golgi apparatus \cite{torii2024myelination}.

These findings suggest that TH2 neurons exhibit higher activity, stronger neuronal proliferation and survival capabilities, and enhanced synaptic transmission and protein synthesis, all contributing to the more efficient signaling capabilities of this subtype.

\section{Conclusion}
\label{sec:conclusion}
In this study, we propose Adaptive Multi-Scale Manifold Embedding (AMSME), a robust two-stage embedding framework designed to address the limitations of traditional manifold embedding methods. By introducing ordinal-based distances, AMSME theoretically overcomes the shortcomings of traditional distance metrics, which are prone to the curse of dimensionality in high-dimensional spaces and exhibit low inter-cluster discriminability. Additionally, the adaptive local neighborhood selection mechanism enables AMSME to simultaneously preserve both local and global structures across points with varying densities, enhancing cluster separability and ensuring robustness against data noise and heterogeneity. The two-phase embedding process provides distinct results: one focusing on intra-cluster structure and the other emphasizing inter-cluster separation.

Experimental results on real-world datasets demonstrate that AMSME outperforms state-of-the-art methods, including t-SNE, UMAP, and PaCMAP, in terms of both inter-cluster separation and intra-cluster local structure preservation. Specifically, the two stage of AMSME improve clustering accuracy by 6\% and 10.9\%, respectively, and show significant advantages in high-neighborhood-size KNN classification algorithms. Furthermore, AMSME exhibits multi-resolution analysis capabilities, identifying novel subtypes in scRNA-seq datasets and revealing their biological differences. These strengths and functionalities highlight AMSME's potential in practical applications such as social network analysis, where it effectively detects community structures and their dynamic changes.

However, AMSME still has some limitations. For instance, its computational complexity may require further optimization for ultra-large-scale datasets. Future work will focus on the following aspects: first, designing distributed algorithms to support large-scale computations; and second, extending AMSME to dynamic or streaming data scenarios to enable real-time data analysis. We believe that AMSME, as a versatile tool for high-dimensional data analysis and visualization, will play an increasingly important role in a wide range of applications.

\ifCLASSOPTIONcaptionsoff
  \newpage
\fi
\bibliographystyle{IEEEtran}
\bibliography{reference}%
\end{document}